\documentclass{article}
\usepackage{amsthm,amsmath,amssymb,natbib,graphicx,url}
\usepackage{color}
\usepackage{times}
\usepackage{comment}
\usepackage{anysize}
\marginsize{3cm}{3.7cm}{3cm}{3.7cm}
\newcommand{\hf}{\hat{f}}

{%
\par\noindent{\bfseries\upshape
  #1\ }
}{%
\hfill\BlackBox\\[2mm]
}

\newcommand{\Comments}{0}
\newcommand{\mynote}[2]{\ifnum\Comments=1\textcolor{#1}{#2}\fi}
\definecolor{blue}{rgb}{0,0.0,0.8}

\newcommand{\Regret}{\operatorname{Regret}}
\newcommand{\Reward}{\operatorname{Reward}}

\newcommand{\q}{q}   %
\newcommand{\qv}{v}  %
\renewcommand{\ng}{\theta}

\newcommand{\G}{\mathcal{G}}
\newcommand{\W}{\mathcal{W}}
\newcommand{\opteps}{\eps}
\newcommand{\approxeps}{\hat{\eps}}
\newcommand{\Rad}{\{-1, 1\}}

\DeclareMathOperator*{\argmin}{arg\,min}
\DeclareMathOperator*{\argmax}{arg\,max}

\def \eps {\epsilon}

\newcommand{\eqr}[1]{\eqref{eq:#1}}

\newcommand {\norm}[1]{\ensuremath{\| #1 \|}}
\newcommand {\dnorm}[1]{\ensuremath{\| #1 \|_*}}

\newcommand {\ang}[1]{\left\langle#1\right\rangle}

\newcommand{\BO}{\mathcal{O}}

\newcommand{\R}{\ensuremath{\mathbb{R}}}

\newcommand{\Hilbert}{\ensuremath{\mathcal{H}}}

\newcommand{\grad}{\triangledown}

\newcommand{\abs}[1]{|#1|}

\newcommand {\bef}[2]{\exp\left(\frac{#1}{#2}\right)}

\DeclareMathOperator*{\E}{\mathbb{E}}

\newcommand{\qqand}{\qquad \text{and} \qquad}

\newcommand{\h}{\frac{1}{2}}

\makeatletter
\@ifundefined{theorem}{%
\newtheorem {theorem}{Theorem}}{}
\@ifundefined{lemma}{%
\newtheorem {lemma}[theorem]{Lemma}}{}
\@ifundefined{corollary}{%
\newtheorem {corollary}[theorem]{Corollary}}{}

\newcommand{\bu}{\boldsymbol{u}}

\newcommand{\bw}{\boldsymbol{w}}

\renewcommand{\ss}{\subseteq}

\newcommand{\coltonly}[1]{}
\newcommand{\arxivonly}[1]{#1}

\title{
Unconstrained Online Linear Learning in Hilbert Spaces:\\
 Minimax Algorithms and Normal Approximations}

\author{
H. Brendan McMahan\thanks{Both authors thank Jacob Abernethy for insightful feedback on this work.} \\
\small{Google} \\
\texttt{\small mcmahan@google.com}
\and
Francesco Orabona\footnotemark[1] \\ 
\small{TTI-C} \\
\texttt{\small orabona@ttic.edu}
}

\renewenvironment{abstract}
 {\small
  \begin{center}
  \bfseries \abstractname\vspace{-.5em}\vspace{0pt}
  \end{center}
  \list{}{%
    \setlength{\leftmargin}{20mm}%
    \setlength{\rightmargin}{\leftmargin}%
  }%
  \item\relax}
 {\endlist}

\begin{document}

\maketitle

\begin{abstract}
  We study algorithms for online linear optimization in Hilbert
  spaces, focusing on the case where the player is unconstrained.  We
  develop a novel characterization of a large class of minimax
  algorithms, recovering, and even improving, several previous results
  as immediate corollaries.  Moreover, using our tools, we develop an
  algorithm that provides a regret bound of $\BO\Big(U \sqrt{T \log( U
    \sqrt{T} \log^2 T +1)}\Big)$, where $U$ is the $L_2$ norm of an arbitrary
  comparator and both $T$ and $U$ are unknown to the player. This
  bound is optimal up to $\sqrt{\log \log T}$ terms.  When $T$ is
  known, we derive an algorithm with an optimal regret bound (up to
  constant factors).  For both the known and unknown $T$ case, a Normal
  approximation to the conditional value of the game proves to be the
  key analysis tool.
\end{abstract}

\coltonly{
}

\section{Introduction}
\noindent

The online learning framework provides a scalable and flexible approach for modeling a wide range of prediction problems, including classification, regression, ranking, and portfolio management. Online algorithms work in rounds, where at each round a new instance is given and the algorithm makes a prediction. Then the environment reveals the label of the instance, and the learning algorithm updates its internal hypothesis. The aim of the learner is to minimize the cumulative loss it suffers due to its prediction error.

Research in this area has mainly focused on designing new prediction strategies and proving theoretical guarantees for them.
However, recently, minimax analysis has been proposed as a general tool to design optimal prediction strategies~\citep{rakhlin12relax,rakhlin13localization,mcmahan13minimax}.
The problem is cast as a sequential multi-stage zero-sum game between the player (the learner) and an adversary (the environment), providing the \emph{optimal} strategies for both. In some cases the value of the game can be calculated exactly in an efficient way~\citep{abernethy08}, in others upper bounds on the value of the game (often based on the sequential Rademacher complexity) are used to construct efficient algorithms with theoretical guarantees~\citep{rakhlin12relax}.

While most of the work in this area has focused on the setting where
the player is constrained to a bounded convex set~\citep{abernethy08} (with
the notable exception of \citet{mcmahan13minimax}), in this work we
are interested in the general setting of unconstrained online learning with linear
losses in Hilbert spaces. In Section~\ref{sec:newanalysis}, extending the work of
\citet{mcmahan13minimax}, we provide novel and general sufficient conditions to be
able to compute the exact minimax strategy for both the player and
the adversary, as well as the value of the game.
In particular, we show that under these conditions the optimal play of the adversary is
\emph{always orthogonal or always parallel to the sum of his previous plays}, while \emph{the optimal play of the player is always parallel}.
On the other hand, for some cases where the exact minimax strategy is
hard to characterize, we introduce a new relaxation procedure based on
a Normal approximation.  In the particular application of interest, we
show the relaxation is strong enough to yield an optimal regret bound,
up to constant factors.

In Section~\ref{sec:palgs}, we use our new tools to recover and extend
previous results on minimax strategies for linear online learning,
including results for bounded domains.  In fact, we show how to obtain
a family of minimax strategies that smoothly interpolates between the
minimax algorithm for a bounded feasible set and a minimax optimal
algorithm in fact equivalent to unconstrained gradient descent.
We emphasize that \emph{all} the algorithms from this family are
exactly minimax optimal,\footnote{In this work, we use the term
  ``minimax'' to refer to the exact minimax solution to the zero sum
  game, as opposed to algorithms that only achieve the minimax optimal
  rate up to say constant factors.} in a sense we will make precise in
the next section.
Moreover, if you are allowed to play outside of the comparator set, we
show that some members of this family have a non-vacuous regret bound
for the unconstrained setting, while remaining optimal for the constrained
one.

When studying unconstrained problems, a natural question is how small we can make the dependence of the regret bound on $U$, the $L_2$ norm of an arbitrary comparator point, while still maintaining a $\sqrt{T}$ dependency on the time horizon.  The best algorithm from the above family achieves $\Regret(U) \leq \h(U^2 + 1)\sqrt{T}$.
\citet{streeter12unconstrained} and \citet{orabona13dimfree} show  it is
possible to reduce the dependence on $U$ to $\BO(U \log U T)$.  In order to improve on this, in Section~\ref{sec:expg} we apply our techniques to analyze a strategy, based on a Normal potential function, that gives a regret bound of $\BO\Big(U \sqrt{T \log( U \sqrt{T} \log^2 T +1)}\Big)$ where $U$ is the $L_2$ norm of a comparator, and both $T$ and $U$ are unknown. This
bound is optimal up to $\sqrt{\log \log T}$ terms.
Moreover, when $T$ is known, we propose an algorithm based on a similar potential function that is optimal up to constant terms.
This solves the open problem posed in those papers,
matching the lower bound for this problem.
Table~\ref{tab:bounds} summarizes the regret bounds we prove, along with those for related algorithms.

Our analysis tools for both known-$T$ and unknown horizon algorithms
rest heavily on the relationship between the reward (negative loss)
achieved by the algorithm, potential functions that provide a
benchmark for the amount of reward the algorithm should have, the
regret of the algorithm with respect to a post-hoc comparator $u$, and
the conditional value of the game.  These are familiar concepts from
the literature, but we summarize these relationships and provide some
modest generalizations in Section~\ref{sec:potential}.

\newcommand{\T}{\rule{0pt}{2.2ex}}
\newcommand{\ResultTable}{%
\begin{table}
\begin{small}
\renewcommand{\arraystretch}{1.3}
\begin{tabular}{l l l l}
\multicolumn{3}{l}{
\noindent
\textbf{Regret bounds for known-$T$ algorithms}}\\
\hline
(A) & \T Minimax Regret & $\sqrt{T}$ for $u \in \W$, $\BO(T)$ otherwise
   & \cite{abernethy08}\\
(B) & \T OGD, fixed $\eta$ & $\h(1 + U^2)\sqrt{T}$
   & E.g., \citet{shalevshwartz12online}\\
(C) & $pq$-Algorithm & $\big(\frac{1}{p} + \frac{1}{q}U^q\big) \sqrt{T}$
   & Cor.~\ref{cor:palg}, which also covers (A) and (B)\\
(D) & Reward Doubling & $\BO\big(U \sqrt{T} \log(d (U + 1) T) \big)$
   & \citet{streeter12unconstrained}\\
(E) & Normal Potential, $\eps=1$
   & $\BO\left(U \sqrt{T \log \big(U T + 1\big)}\right)$
   & Theorem~\ref{thm:expgknownT}\\
(F) & Normal Potential, $\eps=\sqrt{T}$
   & $\BO\left( (U+1)\sqrt{T\log(U + 1)}\right)$
   & Theorem~\ref{thm:expgknownT}\\
\multicolumn{3}{l}{
\noindent
\textbf{Regret bounds for adaptive algorithms for unknown $T$}}\\
\hline
(G) & \T Adaptive FTRL/RDA & $(1 + \h U^2)\sqrt{T}$
   & \citet{Shalev-Shwartz07,xiao09dualaveraging} \\
(H) & Dim. Free Exp. Grad. &
   $\BO\big(U \sqrt{T} \log(U T + 1)\big)$
   &\citet{orabona13dimfree} \\
(I) & AdaptiveNormal
   & $\BO^*\left(U \sqrt{T \log\big(U T + 1\big)}\right)$
   & Theorem~\ref{thm:expgadapt} \\
\end{tabular}
\caption{\arxivonly{Summary of regret bounds for unconstrained linear optimization.}  Here $U = \norm{u}$ is the norm of a comparator, with $U$ unknown to the algorithm.  We let $\W = \{w : \norm{w} \leq 1\}$; the adversary plays gradients with $\norm{g_t} \le 1$.  (A) is minimax optimal for regret against points in $\W$, and always plays points from $\W$.  The other algorithms are unconstrained.  Even though (A) is minimax optimal for regret, other algorithms (e.g. (B)) offer strictly better bounds for arbitrary $U$.  (C) corresponds to a family of minimax optimal algorithms where $\frac{1}{p} + \frac{1}{q} = 1$; $p=2$ yields $(B)$ and as $p\rightarrow 1$ the algorithm becomes (A); Corollary~\ref{cor:palg} covers (A) exactly.  Only (D) has a dependence on $d$, the dimension of $\Hilbert$.  The $\BO^*$ in (I) hides an additional $\log^2(T+1)$ term inside the $\log$.
  }
\vspace{-0.15in}
\label{tab:bounds}
\end{small}
\end{table}}

\ResultTable

\section{Notation and Problem Formulation}
Let $\Hilbert$ be a Hilbert space with inner product $\langle \cdot, \cdot \rangle$. The associated norm is denoted by $\|\cdot\|$, i.e. $\|x\|=\sqrt{\langle x, x \rangle}$.
Given a closed and convex function $f$ with domain $S \ss \Hilbert$, we will denote its Fenchel conjugate by $f^* : \Hilbert \to \R$ where $f^*(u) = \sup_{v\in S} \bigl( \langle v , u \rangle - f (v)\bigr)$.

We consider a version of online linear optimization, a standard game for studying repeated decision making. On each of a sequence of rounds, a \emph{player} chooses an action $w_t \in \Hilbert$, an \emph{adversary} chooses a linear cost function
$g_t \in \G \subseteq \Hilbert$, and the player suffers loss $\langle w_t,
g_t \rangle$. For any sequence of plays $w_1, \dots w_T$ and $g_1, \dots, g_T$, we define the \emph{regret} against a comparator $u$ in the standard way:
\begin{eqnarray*}
\Regret(u) & \equiv & \sum_{t=1}^T \langle g_t ,w_t - u \rangle~.
\end{eqnarray*}
This setting is general enough to cover the cases of online learning in, for example, $\R^d$, in the vector space of matrices, and in a RKHS.
We also define the \emph{reward} of the algorithm, which is the earnings (or negative losses) of the player throughout the game:
\begin{eqnarray*}
\Reward & \equiv &  \sum_{t=1}^T \langle -g_t, w_t \rangle~.
\end{eqnarray*}
We write $\ng_t \equiv -g_{1:t}$, where we use the compressed
summation notation $g_{1:t} \equiv \sum_{s=1}^T g_s$.

\paragraph{The Minimax View}
It will be useful to consider a full game-theoretic characterization
of the above interaction when the number of rounds $T$ is known to
both players.  This approach that has received significant recent
interest \citep{abernethy08, binning, abernethy2010repeated,
  abernethy2008optimal, streeter12unconstrained}.

In the constrained setting, where the comparator vector $\bu \in \W$,  we have that the \emph{value of the game}, that is the regret when both the player and the adversary play optimally, is
\begin{align*}
V &\equiv \min_{w_1 \in \Hilbert} \max_{g_1 \in \G} \cdots
    \min_{w_T\in \Hilbert} \max_{g_T \in \G}
    \left( \sup_{\bu \in \W} \  \sum_{t=1}^T \langle w_t-u , g_t \rangle \right) \\
&= \min_{w_1 \in \Hilbert} \max_{g_1 \in \G} \cdots
    \min_{w_T\in \Hilbert} \max_{g_T \in \G}
    \left( \sum_{t=1}^T \langle w_t, g_t \rangle + \sup_{u \in \W} \langle u, \ng_T \rangle \right) \\
&= \min_{w_1 \in \Hilbert} \max_{g_1 \in \G} \cdots
    \min_{w_T\in \Hilbert} \max_{g_T \in \G}
    \left( \sum_{t=1}^T \langle w_t, g_t \rangle + B(\ng_T) \right),
\end{align*}
where
\begin{equation}\label{eq:stdregB}
B(\ng) = \sup_{w \in \W} \ang{w, \ng}~.
\end{equation}

Following \citet{mcmahan13minimax}, we generalize the game in terms of a generic
convex \emph{benchmark function} $B : \Hilbert \to \R$, instead of using the definition \eqref{eq:stdregB}. This allows us to analyze the constrained and unconstrained setting in a unified way.
Hence, the value of the game is the difference between the benchmark
reward $B(\ng_T)$ and the actual reward achieved by the player (under
optimal play by both parties).  Intuitively, viewing the units of
loss/reward as dollars, $V$ is the amount of starting capital we need
(equivalently, the amount we need to borrow) to ensure we end the game
with $B(\ng_T)$ dollars.  The motivation for defining the game in
terms of an arbitrary $B$ is made clear in the next section: It will
allow us to derive Regret bounds in terms of the Fenchel conjugate of
$B$.

We define inductively the \emph{conditional value of the game} after $g_1,
\dots, g_t$ have been played by
\[
 V_t(\ng_t) = \min_{w \in \Hilbert} \max_{g \in \G} \left(\langle g, w \rangle + V_{t+1}(\ng_t - g)\right)
\qquad \text{with} \qquad V_T(\ng_T) = B(\ng_T)~.
\]
Thus, we can view the notation $V$ for the value of the game as shorthand for
$V_0(\mathbf{0})$.
Under minimax play by both players, unrolling the previous equality, we have
$
\sum_{s=1}^t \langle g_s, w_s \rangle + V_t(-g_{1:t}) = V,
$
or for $t=T$,
\begin{equation}\label{eq:mmrt}
\Reward = \sum_{t=1}^T \ang{-g_t, w_t} = B(\ng_T) - V~.
\end{equation}
We also have that, given the conditional value of the game, a minimax-optimal strategy is
\begin{equation}
\label{eq:opt_strat_from_cond_value}
w_{t+1} = \argmin_w \max_{g \in \G} \ \langle g, w \rangle + V_{t+1}(\ng_t - g)~.
\end{equation}
\citet[Cor. 2]{mcmahan13minimax} showed that in the unconstrained
case, $V_t$ is a smoothed version of $B$, where the smoothing comes
from an expectation over future plays of the adversary.  In this
work, we show that in some cases (Theorem~\ref{thm:exact_minmax}) we can
find a closed form for $V_t$ in terms of $B$, and in fact the solution
to \eqr{opt_strat_from_cond_value} will simply be the gradient of
$V_{t}$, or equivalently, an FTRL algorithm with regularizer $V_t^*$.
On the other hand, to derive our main results, we face a case
(Theorem~\ref{thm:normapprox}) where $V_t$ is generally \emph{not}
expressible in closed form, and the resulting algorithm does
\emph{not} look like FTRL.  We solve the first problem by using a
Normal approximation to the adversary's future moves, and we solve the
second by showing \eqr{opt_strat_from_cond_value} can still be solved
in closed form with respect to this approximation to $V_t$.

\section{Potential Functions and the Duality of Reward and Regret}\label{sec:potential}

In the present section we will review some existing results in online learning theory as well as provide a number of mild generalizations for our purposes.
\emph{Potential functions} play a major role in the design and analysis of online learning algorithms~\citep{Cesa-BianchiL06}.
We will use $\q : \Hilbert \to \R$ to describe the potential, and the key assumptions are that $\q$ should depend solely on the cumulative gradients $g_{1:T}$ and that $\q$ is convex in this argument.\footnote{It is sometimes possible to generalize to potentials $q_t(g_1, \dots, g_t)$ that are functions of each gradient individually.}
Since our aim is adaptive algorithms, we often look at a sequence of changing potential functions $\q_1, \ldots, \q_T$, each of which takes as argument $-g_{1:t}$ and is convex.
These functions have appeared with different interpretations in many
papers, with different emphasis.  They can be viewed as 1) the
conjugate of an (implicit) time-varying regularizer in a Mirror
Descent or Follow-the-Regularized-Leader (FTRL) algorithm~\citep{Cesa-BianchiL06,Shalev-Shwartz07,rakhlin2009lecture}, 2) as proxy for
the conditional value of the game in a minimax setting~\citep{rakhlin12relax}, or 3) a
potential giving a bound on the amount of reward we want the algorithm
to have obtained at the end of round $t$~\citep{streeter12unconstrained,mcmahan13minimax}.

These views are of course closely connected, but can lead to somewhat
different analysis techniques.  Following the last view, suppose we
interpret $\q_t(\ng_t)$ as the desired reward at the end of round $t$,
given the adversary has played $\ng_t = -g_{1:t}$ so far.  Then, if we
can bound our actual final reward in terms of $\q_T(\ng_T)$, we also
immediately get a regret bound stated in terms of the Fenchel
conjugate $q_T^*$.  Generalizing
\citet[Thm. 1]{streeter12unconstrained}, we have the following result (all omitted proofs can be found in the Appendix).
\begin{theorem}\label{thm:rrdual}
  Let $\Psi:\Hilbert \rightarrow \R$ be a convex function. An
  algorithm for the player guarantees
  \begin{equation}\label{eq:rewb}
  \Reward \geq \Psi(-g_{1:T}) - \approxeps
     \quad \quad \quad \textnormal{ for any } g_1, \dots, g_T
  \end{equation}
  for a constant $\approxeps \in \R$ if and only if it
  guarantees
  \begin{equation}\label{eq:regb}
  \qquad \Regret(u) \leq \Psi^*(u) + \approxeps
     \quad \quad \quad \textnormal{ for all } u \in \Hilbert~.
  \end{equation}
\end{theorem}
First we consider the minimax setting, where we define the game in
terms of a convex benchmark $B$.  Then, \eqr{mmrt} gives us an
immediate lower bound on the reward of the minimax strategy for the player
(against any adversary), and so applying Theorem
\ref{thm:rrdual} with $\Psi = B$ gives
\begin{equation}\label{eq:mmrg}
\forall u \in \Hilbert, \qquad \Regret(u) \leq B^*(u) + V~.
\end{equation}
The fundamental point, of which we will make much use, is this: even
if one only cares about the traditional definition of regret, the
study of the minimax game defined in terms of a general comparator
benchmark $B$ may be interesting, as the minimax algorithm for the
player may then give novel bounds on regret.  Note when $B$ is defined
as in \eqref{eq:stdregB}, the theorem implies $\forall u \in \W,\
\Regret(u) \leq V$.
More generally, even for non-minimax algorithms,
Theorem~\ref{thm:rrdual} states that understanding the reward
(equivalently, loss) of an algorithm as a function of the sum of
gradients chosen by the adversary is both necessary and sufficient for
understanding the regret of the algorithm.

Now we consider the potential function view.  The following general
bound for any sequence of plays $w_t$ against gradients $g_t$, for an
arbitrary sequence of potential functions $q_t$, has been used numerous
times (see \citet[Lemma 1]{orabona13dimfree} and references therein).
The claim is that
\begin{equation}\label{eq:regf2}
\Regret(u) \leq \q_T^*(u)
+ \sum_{t=1}^T \left(\q_t(\ng_t) - \q_{t-1}(\ng_{t-1}) + \langle w_t, g_t \rangle \right),
\end{equation}
where we take $\ng_0 = \vec{0}$, and assume $q_0(\vec{0}) = 0$.  In fact, this
statement is essentially equivalent to the argument of \eqref{eq:rewb} and \eqref{eq:regb}.  For intuition, we can view
$\q_t(\ng_t)$ as the amount of money we wish to have available at
the end of round $t$.  Suppose at the end of each round $t$, we borrow
an additional sum $\opteps_t$ as needed to ensure we actually have
$\q_t(\ng_t)$ on hand.  Then, based on this invariant, the amount of
reward we actually have after playing on round $t$ is
$\q_{t-1}(\ng_{t-1}) + \langle w_t, -g_t \rangle $, the money we had
at the beginning of the round, plus the reward we get for playing
$w_t$.  Thus, the additional amount we need to borrow at the end of
round $t$ in order to maintain the invariant is exactly
\begin{equation} \label{eq:epst}
\opteps_t(\ng_{t-1}, g_t) \equiv
\underbrace{\q_t(\ng_t)}_{\text{Reward desired}} - \underbrace{
\big(\q_{t-1}(\ng_{t-1}) + \langle w_t, -g_t \rangle \big)}_{
\text{Reward achieved}},
\end{equation}
recalling $\ng_t = \ng_{t-1} - g_t$.  Thus, if we can find bounds
$\approxeps_t$ such that for all $t$, $\ng_{t-1}$, and $g \in \G$,
\begin{equation}\label{eq:approxepst}
\approxeps_t \geq \opteps_t(\ng_{t-1}, g_t)
\end{equation}
we can re-state \eqref{eq:regf2} as exactly \eqref{eq:regb} with $\Psi
= q_T$ and $\approxeps = \approxeps_{1:T}$.  Further, solving
\eqref{eq:epst} for the per-round reward $\langle w_t, -g_t \rangle$,
summing from $t=1$ to $T$ and canceling telescoping terms gives
exactly \eqref{eq:rewb}.  Not surprisingly, both
Theorem~\ref{thm:rrdual} and \eqref{eq:regf2} can be proved in terms
of the Fenchel-Young inequality.

When $T$ is known, and the $q_t$ are chosen carefully, it is possible
to obtain $\approxeps_t = 0$.  On the other hand, when $T$ is unknown
to the players, typically we will need bounds $\approxeps_t > 0$.  For
example, in both \citet[Thm. 6]{streeter12unconstrained} and
\citet{orabona13dimfree}, the key is showing the sum of these
$\approxeps_t$ terms is always bounded by a constant.
For completeness, we also state
standard results where we interpret $q^*_t$ as a regularizer.

\paragraph{The conjugate regularizer and Bregman divergences}
The updates of many algorithms are based on a time-varying version of the FTRL strategy,
\begin{equation}
\label{eq:ftrl}
w_{t+1} = \nabla \q_t(\ng_t)
        = \argmin_w \ \langle g_{1:t}, w \rangle + q^*_t(w),
\end{equation}
where we view $\q_t^*$ as a time-varying regularizer (see
\citet{OrabonaCCB13} and references therein).  Regret bounds can be
easily obtained using \eqref{eq:regf2} when the regularizers
$q^*_t(w)$ are increasing with $t$, and they are strongly convex
w.r.t. a norm $\dnorm{\cdot}$, using the fact that the potential functions $q_t$ will
be strongly smooth.
Then strong smoothness and particular choice of $w_t$ implies
\begin{equation} \label{eq:ss}
\q_{t-1}(\ng_t) \leq  \q_{t-1}(\ng_{t-1}) - \langle w_t, g_t \rangle+ \h \norm{g_t}^2,
\end{equation}
which leads to the bound
\begin{align*}
 \opteps_t(\ng_t,g_t) = \q_t(\ng_t) - \q_{t-1}(\ng_{t-1}) +  \langle w_t, g_t \rangle
\ \le \
 \q_t(\ng_t) -\q_{t-1}(\ng_t)  + \h \norm{g_t}^2
\ \le\  \h \norm{g_t}^2,
\end{align*}
where the last inequality follows from the fact that if $f(x) \le g(x)$,
then $f^*(y) \ge g^*(y)$ (immediate from the definition of the
conjugate).

When the regularizer $q^*$ is fixed, that is,
$q_t = q$ for all $t$ for some convex function $q$, we get
the approach pioneered by \citet{grove2001general} and \citet{kivinen2001relative}:
\[
 \opteps_t(\ng_t,g_t)
 = \q(\ng_t) - \q(\ng_{t-1}) + \langle w_t, g_t \rangle
 = \q(\ng_t) - \big(\q(\ng_{t-1}) + \langle \nabla q(\ng_{t-1}), g_t \rangle \big)
 = D_\q(\ng_t, \ng_{t-1}),
\]
where $D_q$ is the Bregman Divergence with respect to $q$, and we predict with $w_t = \nabla q(\ng_{t-1})$.

\paragraph{Admissible relaxations and potentials}
We extend the notion of relaxations of the conditional value of the
game of \citet{rakhlin12relax} to the present setting.  We say $\qv_t$
with corresponding strategy $w_t$ is a relaxation of $V_t$
if
\begin{align}
\forall \ng,& \qquad \qv_T(\ng) \geq B(\ng) \label{eq:hVb} && \text{and}\\
\forall t \in \{0, \dots, T-1\}, g \in \G, \ng \in \Hilbert,
  & \qquad \qv_t(\ng) + \approxeps_{t+1}
   \geq \langle g, w_{t+1} \rangle + \qv_{t+1}(\ng - g),\label{eq:repst}
\end{align}
for constants $\approxeps_t \geq 0$.  This definition matches Eq.~(4)
of \citet{rakhlin12relax} if we force all $\approxeps_t = 0$, but if we
allow some slack $\approxeps_t$, \eqref{eq:repst} corresponds exactly to
\eqref{eq:epst} and \eqref{eq:approxepst}.

Note that \eqref{eq:repst} is invariant to adding a constant to all
$\qv_t$.  In particular, given an admissible $\qv_t$, we can define
$\q_t(\ng) = \qv_t(\ng) - \qv_0(\vec{0})$ so $\q_t(\vec{0}) = 0$ and $q$ satisfies
\eqref{eq:approxepst} with the same $\approxeps_t$ values for which $\qv_t$
satisfies \eqref{eq:repst}.  Or we could define $\q_0(\vec{0}) = 0$ and $\q_t(\ng)
= \qv_t(\ng)$ for $t \geq 1$, and take $\approxeps_1 \leftarrow
\approxeps_1 + \qv_0(\vec{0})$ (or any other way of distributing the
$v_0(\vec{0})$ into the $\approxeps$).
Generally, when $T$ is known we will find working with admissible
relaxations $\qv_t$ to be most useful, while for unknown horizons $T$,
potential functions with $q_0(\vec{0})=0$ will be more natural.

For our admissible relaxations, we have a result that closely mirrors
Theorem~\ref{thm:rrdual}:
\begin{corollary}\label{cor:admbench}
  Let $\qv_0, \dots, \qv_T$ be an admissible relaxation for a
  benchmark $B$.  Then, for any sequence $g_1, \dots, g_T$, for any
  $w_t$ chosen so \eqref{eq:repst} and \eqref{eq:hVb} are satisfied, we have
  \[
  \Reward \geq B(\ng_T) - \qv_0(0) - \approxeps_{1:T}
  \qqand
  \Regret(u) \leq B^*(u) +\qv_0(0) + \approxeps_{1:T}~.
  \]
\end{corollary}
\begin{proof}
  For the first statement, re-arranging and summing \eqref{eq:repst} shows
  $\Reward_t \ge \qv_t(\ng_t) -\approxeps_{1:t} - \qv_0(0)$ and so
  final $\Reward \geq B(\ng) - \qv_0(0) - \approxeps_{1:T}$; the
  second result then follows from Theorem~\ref{thm:rrdual}.
\end{proof}
The regret bound corresponds to \eqref{eq:mmrg}; in particular, if we take
$\qv_t$ to be the conditional value of the game, then \eqref{eq:hVb} and \eqref{eq:repst} hold with equality with all
$\approxeps_t=0$.  Note if we define $B$ as in \eqref{eq:stdregB}, the
regret guarantee becomes
$
\forall u \in \W,\ \Regret(u) \leq \qv_0(0) + \approxeps_{1:T},
$
analogous to \citep[Prop. 1]{rakhlin12relax} when $\approxeps_{1:T} = 0$.

\paragraph{Deriving algorithms}
Consider an admissible relaxation $\qv_t$.  Given
the form of the regret bounds we have proved, a natural strategy is to
choose $w_{t+1}$ so as to minimize $\approxeps_{t+1}$, that is,
\begin{equation}
w_{t+1} %
        = \argmin_w \max_{g \in \G} \  \qv_{t+1}(\ng_t - g) - \qv_t(\ng_t) + \langle g, w \rangle
        = \argmin_w \max_{g \in \G} \ \langle g, w \rangle + \qv_{t+1}(\ng_t - g), \label{eq:mmq}
\end{equation}
following \citet[Eq. (5)]{rakhlin12relax}, \citet{rakhlin13localization}, and
\citet[Eq. (8)]{streeter12unconstrained}.  We see that
$\qv_{t+1}$ is standing in for the conditional value of the game in \eqref{eq:opt_strat_from_cond_value}.
Since additive constants do not impact the argmin, we could also
replace $\qv_t$ with a potential $\q_t$, say $\q_t(\ng) = \qv_t(\ng) -
\qv_0(0)$.

\section{Minimax Analysis Approaches for Known-Horizon Games}
\label{sec:newanalysis}

In general, the problem of calculating the conditional value of a game
$V_t(\ng)$ is hard.  And even for a known potential, deriving an
optimal solution via \eqref{eq:mmq} is also in general a hard problem.
When the player is unconstrained, we can simplify the computation of
$V_t$ and the derivation of optimal strategies.  For example,
following ideas from \citet{mcmahan13minimax},
\[
 \opteps_t(\theta_t) = \max_{p \in \Delta(\G), \E_{g \sim p}[g] = 0} \ \E_{g \sim p}[
   \q_{t+1}(\ng_t - g)] - \q_t(\ng_t),
\]
where $\Delta(\G)$ is the set of probability distributions on $\G$.
\citet{mcmahan13minimax} shows that in some cases is possible to
easily calculate this maximum, in particular when $\G = [-G, G]^d$ and
$\q_t$ decomposes on a per-coordinate spaces (that is, when the
problem is essentially $d$ independent, one-dimensional problems).

In this section we will state two quite general cases where we can
obtain the exact value of the game, even though the problem
does not decompose on a per coordinate basis. Note that in both cases
the optimal strategy for $w_{t+1}$ will be in the direction of $\ng_t$.

We study the game when the horizon $T$ is known, with a benchmark
function of the form $B(\ng)=f(\norm{\ng})$ for an increasing convex
function $f\!:\![0,+\infty] \rightarrow \R$ (which ensures $B$ is
convex).  Note this form for $B$ is particularly natural given our
desire to prove results that hold for general Hilbert spaces.  We will
then be able to derive regret bounds using Theorem~\ref{thm:rrdual},
and the following technical lemma:
\begin{lemma}
\label{lemma:fenchel_norm}
  Let $B(\ng) = f(\norm{\ng})$ for $f\!: \R \rightarrow\, (\!-\!\infty,+\infty]$ even.  Then, $B^*(u) = f^*(\norm{u})$.
\end{lemma}
Recall that $f$ is even if $f(x) = f(-x)$.
Our key tool will be a careful study of the one-round version of this
game.  For this section, we let $h\!:\!\R \rightarrow
\R$ be an even convex function that is
increasing on $[0, \infty]$, $\G=\{g: \|g\| \leq G\}$, and $d$ the
dimension of $\Hilbert$.  We consider the one-round game
\begin{equation}\label{eq:oneroundH}
H \equiv \min_{w} \max_{g \in \G}\  \langle w, g \rangle + h(\norm{\theta - g})~,
\end{equation}
where $\theta \in \Hilbert$ is a fixed parameter.
For results regarding this game, we let $H(w, g)=\langle w, g \rangle
+ h(\norm{\theta - g})$, $w^*=\argmin_w \max_{g \in \G} H(w,g)$, and
$g^*=\argmax_{g \in \G} H(w^*,g)$. Also, let
$\hat{\ng}=\frac{\ng}{\norm{\ng}}$ if $\norm{\ng} \neq 0$, and
$\vec{0}$ otherwise.

\subsection{The case of the orthogonal adversary}
\label{sec:orthogonal_adv}
Let $B(\ng)=f(\norm{\ng})$ for an increasing convex function $f:[0, \infty] \rightarrow \R$, and define 
\coltonly{$f_t(x)=f(\sqrt{x^2+G^2 (T-t)})$.}
\arxivonly{\[f_t(x)=f(\sqrt{x^2+G^2 (T-t)})~.\]}
Note that $f_t(\norm{\ng})$ can be viewed as a smoothed version of $B(\ng)$, since
$\sqrt{\norm{\ng}^2 + C}$ is a smoothed version of $\norm{\ng}$ for a
constant $C > 0$. Moreover, $f_0(\norm{\ng})=B(\ng)$.

Our first key result is the following:
\begin{theorem}\label{thm:exact_minmax}
  Let the adversary play from $\G=\{g: \|g\| \leq
  G\}$ and assume all the $f_t$ satisfy
  \begin{equation}\label{eq:strong}
    \min_{w} \max_{g \in \G} \ \ang{ w, g}  + f_{t+1}(\norm{\ng - g}) 
        = f_{t+1}\left(\sqrt{\|\ng\|^2 + G^2}\right)~.
  \end{equation}
  Then the value of the game is $f(G\sqrt{T})$, the conditional value
  is $V_t(\ng) = f_t(\norm{\ng}) = f_{t+1}(\sqrt{\norm{\ng}^2 + G^2})$, and the optimal strategy can be
  found using \eqref{eq:mmq} on $V_t$.

  Further, a sufficient condition for \eqref{eq:strong} is that $d>1$, $f$ is twice differentiable, and $f''(x) \leq f'(x)/x$,
  for all $x > 0$.
  In this case we also have that the minimax optimal strategy is 
  \begin{equation}
  \label{eq:opt_strat_orthogonal}
  w_{t+1} = \nabla V_{t}(\ng_t) 
   = \ng_t \frac{f'(\sqrt{\norm{\ng_t}^2 + G^2(T-t)})}{\sqrt{\norm{\ng_t}^2 + G^2(T-t)}}~.
  \end{equation}
\end{theorem}
In this case, the minimax optimal strategy
\eqref{eq:opt_strat_orthogonal} is equivalent to the FTRL strategy in
\eqref{eq:ftrl} with the time varying regularizer $V^*_{t}(w)$.
The key lemma needed for the proof is the following:
\begin{lemma}\label{lem:mm:three}
  Consider the game of \eqref{eq:oneroundH}.  Then, if $d > 1$, $h$
  is twice differentiable, and $h''(x)\leq \frac{h'(x)}{x}$ for $x >
  0$, we have:
\begin{align*}
H =  h\left(\sqrt{\|\theta\|^2 + G^2}\right)
\qqand
w^* = \frac{\ng}{\sqrt{\|\theta\|^2 + G^2}} h'\left(\sqrt{\|\theta\|^2 + G^2}\right)~.
\end{align*}
Any $g^*$  such that $\ang{\ng,g^*}=0$ and $\norm{g^*}=G$ is a minimax play for the adversary.
\end{lemma}
We defer the proofs to the Appendix (of the proofs in the appendix, the proof of Lemmas ~\ref{lem:mm:three} and \ref{lem:mm:four} are perhaps the most important and instructive).
Since the best response of the adversary is
always to play a $g^*$ orthogonal to $\ng$, we call this the case of
the orthogonal adversary.

\subsection{The case of the parallel adversary, and Normal approximations}\label{sec:parallel_adv}
We analyze a second case where \eqref{eq:oneroundH} has closed-form solution,
and hence derive a class of games where we can cleanly state the value
of the game and the minimax optimal strategy.  The results
of~\cite{mcmahan13minimax} can be viewed as a special case of the
results in this section.

First, we introduce some notation.  We write $\tau \equiv T - t$ when
$T$ and $t$ are clear from context.  We write $r \sim \{-1, 1\}$ to
indicate $r$ is a Rademacher random variable, and $r_\tau \sim \{-1,
1\}^\tau$ to indicate $r_\tau$ is the sum of $\tau$ IID Rademacher
random variables.  Let $\sigma = \sqrt{\pi/2}$.  We write $\phi$ for a
random variable with distribution $N(0, \sigma^2)$, and similarly
define $\phi_\tau \sim N(0, (T - t) \sigma^2)$.  Then, define
\begin{equation} \label{eq:def_ht}
  f_t(x) = \E_{r_\tau \sim \Rad^\tau}\left[
      f\left(\abs{x + r_\tau G}\right)\right]
  \qqand
  \hf_t(x) = \E_{\phi_\tau \sim N(0, \tau \sigma^2)}\left[
     f\left(\abs{x + \phi_\tau G\right})\right],
\end{equation}
and note $B(\ng) = f_T(\norm{\ng}) = \hf_T(\norm{\ng})$ since $\phi_0$
and $r_0$ are always zero.  
These functions are exactly \emph{smoothed} version of the
function $f$ used to define $B$. With these definitions, we can now state:
\begin{theorem}\label{thm:normapprox}
  Let $B(\ng)=f(\norm{\ng})$ for an increasing convex function $f:[0,\infty]
  \rightarrow \R$, and let the adversary play from $\G=\{g: \|g\| \leq
  G\}$.  Assume $f_t$ and $\hf_t$ as in \eqref{eq:def_ht} for all $t$.
  If all the $f_t$ satisfy
  \begin{equation}\label{eq:strong2}
    \min_{w} \max_{g \in \G} \ \langle w, g \rangle + f_{t+1}(\norm{\theta - g}) 
     = \E_{r \sim \{-1, 1\}} \big[f_{t+1}\left(\|\theta\| + r G\right)\big],
  \end{equation}
  then $V_t(\ng) = f_t(\norm{\ng})$ is exactly the conditional
  value of the game, and \eqref{eq:mmq} gives the minimax optimal strategy:
  \begin{equation}
  \label{eq:opt_strat_orthogonal}
  w_{t+1}=\hat{\ng}\frac{f_{t+1}\left(\|\ng\| + G\right)-f_{t+1}\left(\|\ng\| - G\right)}{2G}~.
  \end{equation}
  Similarly, suppose the $\hf_t$ satisfy the equality \eqref{eq:strong2} (with $\hf_t$ replacing $f_t$).  
  Then $\q_t(\ng) =
  \hf_t(\norm{\ng})$ is an admissible relaxation of $V_t$, satisfying
  \eqref{eq:repst} with $\approxeps_{t} = 0$, using $w_{t+1}$ based on
  \eqref{eq:mmq}.
  Further, a sufficient condition for \eqref{eq:strong2} is that $d=1$, or $d>1$, the $f_t$
  (or $\hf_t$, respectively) are twice differentiable, and satisfy and
  $f_t''(x) \geq f_t'(x)/x$ for all $x > 0$.
\end{theorem}
Contrary to the case of the orthogonal adversary, the strategy in
\eqref{eq:opt_strat_orthogonal} \emph{cannot} easily be interpreted as an
FTRL algorithm.
The proof is based on two lemmas. The first provides the key tool in supporting the Normal relaxation:
\begin{lemma}\label{lemma:normap}
  Let $f: \R \rightarrow \R$ be a convex function and $\sigma^2=\pi/2$.  Then,
  \[
  \E_{g \sim \{-1, 1\}} [f(g)] \leq \E_{\phi \sim N(0, \sigma^2)}[f(\phi)]~.
  \]
\end{lemma}
\begin{proof}
  First observe that $\E[(\phi - 1) \mathbf{1}\{\phi > 0\}] = 0$ and $\E[(\phi + 1) \mathbf{1}\{\phi < 0\}] = 0$ by our choice of $\sigma$.  We will use two lower bounds on the function $f$, which follow from convexity:
  \begin{equation*}
    f(x) \geq f(1) + f'(1)(x - 1)   \qqand
    f(x) \geq  f(-1) + f'(-1)(x + 1)~.
  \end{equation*}
  Writing out the value of $\E[f(\phi)]$ explicitly we have
  \begin{eqnarray*}
    \E[f(\phi)]
    & = & \E[f(\phi) \mathbf{1}\{\phi < 0\}]
      + \E[f(\phi) \mathbf{1}\{\phi > 0\}] \\
    & \geq & \E[(f(-1) + f'(-1)(\phi + 1)) \mathbf{1}\{\phi < 0\}]
      + \E[(f(1) + f'(1)(\phi - 1)) \mathbf{1}\{\phi > 0\}] \\
    & = & \frac{f(-1) + f(1)}{2} + f'(-1)\E[(\phi + 1)\mathbf{1}\{\phi < 0\}]
     + f'(1)\E[(\phi - 1) \mathbf{1}\{\phi < 0\}]~.
  \end{eqnarray*}
  The latter two terms vanish, giving the stated inequality.
\end{proof}

The second lemma is used to prove the sufficient condition by solving
the one-round game; again, the proof is deferred to the Appendix.
Note that functions of the form $h(x)=g(x^2)$, with $g$ convex always satisfies the conditions of the following Lemma.
\begin{lemma}\label{lem:mm:four}
  Consider the game of \eqref{eq:oneroundH}.  Then, if $d = 1$, or if $d >
  1$, $h$ is twice differentiable, and $h''(x) > \frac{h'(x)}{x}$ for
  $x > 0$, then
  \begin{align*}
    H =  \frac{h\left(\|\ng\| + G\right)+h\left(\|\ng\| - G\right)}{2}
    \qqand
    w^* = \hat{\ng}\frac{h\left(\|\ng\| + G\right)-h\left(\|\ng\| - G\right)}{2G}.
\end{align*}
Any $g^*$ that satisfies $|\langle \ng,g^* \rangle|=G \norm{\ng}$ and
$\norm{g^*}=$G is a minimax play for the adversary.
\end{lemma}
The adversary can always play $g^* = G\frac{\ng}{\norm{\ng}}$ when
$\ng \neq \mathbf{0}$, and so we describe this as the case of the
parallel adversary.  In fact, inductively this means that \emph{all}
the adversary's plays $g_t$ can be on the same line, providing
intuition for the fact that this lemma also applies in the
1-dimensional case.

Theorem~\ref{thm:normapprox} provides a recipe to produce suitable relaxations $\q_t$ which may, in certain cases, exhibit nice closed form solutions. The interpretation here is that a ``Gaussian adversary'' is stronger than one playing from the set $[-1,1]$ which leads to IID Rademacher behavior, and this allows us to generate such potential functions via Gaussian smoothing. In this view,
note that our choice of $\sigma^2$ gives $\E_\phi[\abs{\phi}] = 1$.

\section{A Power Family of Minimax Algorithms}\label{sec:palgs}
We analyze a family of algorithms based on potentials $B(\ng) =
f(\norm{\ng})$ where $f(x) = \frac{W}{p} \abs{x}^p$ for parameters $W > 0$ and
$p \in [1, 2]$, when the dimension is at least two.  This is reminiscent
of $p$-norm algorithms~\citep{Gentile03}, but the connection is superficial---the norm
we use to measure $\ng$ is always the norm of our Hilbert space.  Our
main result is:
\begin{corollary}\label{cor:palg}
  Let $d > 1$ and $W>0$, and let $f$ and $B$ be defined as above.  Define $f_t(x) =
  \frac{W}{p}\big(x^2 + (T-t)G\big)^{p/2}$.  Then, $f_t(\norm{\ng})$
  is the conditional value of the game, and the optimal strategy is as
  in Theorem~\ref{thm:exact_minmax}.  If $p \in (1, 2]$, letting $q
  \geq 2$ such that $1/p + 1/q = 1$, we have a bound
  \[
  \Regret(u)\
      \ \le \ \frac{1}{W^{q-1} q}\norm{u}^q + \frac{W}{p}\big(G \sqrt{T}\big)^p
      \ \le \ \left(\tfrac{1}{p} + \tfrac{1}{q}\norm{u}^q\right)G\sqrt{T},
   \]
   where the second inequality comes by taking
   $W = (G\sqrt{T})^{1-p}$.
   For all $u$, the bound $\big(\tfrac{1}{p} +
     \tfrac{1}{q}\norm{u}^q\big)G\sqrt{T}$ is minimized by taking
   $p=2$.
   For $p=1$, we have
   \[
   \forall u : \norm{u} \leq W, \quad \Regret(u) \leq W G \sqrt{T}~.
   \]
\end{corollary}
\begin{proof}
  Let $f(x) = \frac{W}{p} \abs{x}^p$ for $p \in [1, 2]$, Then, $f''(x)
  \leq f'(x) / x$, in fact basic calculations show $
  \frac{f'(x)/x}{f''(x)} = \frac{1}{p-1} \geq 1 $ when $p \leq 2$.
  Hence, we can apply Theorem~\ref{thm:exact_minmax}, proving the
  claim on the $f_t$.  The regret bounds can then be derived from
  Corollary~\ref{cor:admbench}, which gives $ \Regret(u) \le f^*(u) +
  f(G \sqrt{T}), $ noting $f^*(u) = \frac{W}{q}\abs{\frac{u}{W}}^q$ when $p > 1$.
  The fact that $p=2$ is an optimal choice in the first bound follows
  from the fact that $\frac{d}{dp} \left(\tfrac{1}{p} +
    \tfrac{1}{q}\norm{u}^q\right) \le 0$ for $p \in (1, 2]$ with $q = \frac{p}{p-1}$.
\end{proof}

The $p=1$ case in fact exactly recaptures the result of
\citet{abernethy08} for linear functions, extending it also to spaces
of dimension equal to two.  The optimal update is
 $w_{t+1} = \grad f_t(\norm{\ng_t}) = W \ng_t/\sqrt{\norm{\ng_t}^2+G^2(T-t)}$.
 In addition to providing a regret bound
for the comparator set $\W = \{u : \norm{u} \leq W\}$, the algorithm
will in fact only play points from this set.

For $p=q=2$, writing $W = \eta$, we have
\[
\Regret(u) \le \frac{1}{2 \eta} \norm{u}^2 + \frac{\eta}{2} G^2 T,
\]
for any $u$.  In this case we see $W=\eta$ is behaving not like the radius
of a comparator set, but rather as a learning rate.  In fact, we have
$
w_{t+1} = \nabla V_t(\ng_t)
  = \eta \ng_t = -\eta g_{1:t},
$
and so we see this minimax-optimal algorithm is in fact
constant-step-size gradient descent.  Taking $\eta = \frac{1}{G
  \sqrt{T}}$ yields $\h (\norm{u}^2 + 1)G \sqrt{T}$.  This result
complements \citet[Thm. 7]{mcmahan13minimax}, which covers the $d=1$
case, or $d>1$ when the adversary plays from $\G = [-1, 1]^d$.

Comparing the $p=1$ and $p>1$ algorithms reveals an interesting fact.
For simplicity, take $G=1$.  Then, the $p=1$ algorithm with $W=1$ is
exactly the minimax optimal algorithm for minimizing regret against
comparators in the $L_2$ ball (for $d > 1$): the value of this game is
$\sqrt{T}$ and we can do no better (even by playing outside of the
comparator set).  However, picking $p > 1$ gives us algorithms that
\emph{will} play outside of the comparator set.  While they cannot do
better than $\sqrt{T}$, taking $G=1$ and $\norm{u}=1$ shows that
\emph{all} algorithms in this family in fact achieve $\Regret(u) \leq
\sqrt{T}$ when $\norm{u} \leq 1$, matching the exact minimax optimal
value. Further, the algorithms with $p > 1$ provide much stronger
guarantees, since they also give non-vacuous guarantees for $\norm{u}
> 1$, and tighter bounds when $\norm{u} < 1$.  This suggests that the
$p=2$ algorithm will be the most useful algorithm in practice,
something that indeed has been observed empirically (given the
prevalence of gradient descent in real applications).
This result also clearly demonstrates the value of studying
minimax-optimal algorithms for different choices of the benchmark $B$,
as this can produce algorithms that are no worse and in some cases
significantly better than minimax algorithms defined in terms of
regret minimization directly (i.e., via \eqr{stdregB}).

The key difference in these algorithms is not how they play against a
minimax optimal adversary for the regret game, but how they play
against non-worst-case adversaries.  In fact, a simple induction based
on Lemma~\ref{lem:mm:three} shows that any minimax-optimal adversary
will play so that $\sqrt{\norm{\ng_t}^2+G^2(T-t)} = G\sqrt{T}$.
Against such an adversary, the $p=1$ algorithm is identical to the
$p=2$ algorithm with learning rate $\eta = \frac{1}{G\sqrt{T}}$.  In
fact, using the choice of $W$ from Corollary~\ref{cor:palg}, all of
these algorithms play identically against a minimax adversary for the
regret game.

\section{Tight Bounds for Unconstrained Learning}
\label{sec:expg}
In this section we analyze algorithms based on benchmarks and
potentials of the form $\exp(\norm{\ng}^2/t)$, and show they lead to a
minimal dependence on $\norm{u}$ in the corresponding regret bounds
for a given upper bound on regret against the origin (equal to the loss of
the algorithm).

First, we derive a lower bound for the known $T$ game.
Using Lemma~\ref{lemma:bound_fenchel} in the Appendix, we can show
that the $B(\ng) = \exp(\norm{\ng}^2/T)$ benchmark approximately corresponds to a
regularizer of the form $\norm{u} \sqrt{T \log (\sqrt{T}
  \norm{u}+1)}$; there is actually some technical challenge here, as the conjugate
$B^*$ cannot be computed in closed form---the given regularizer is an
upper bound.  This kind of regularizer is particularly interesting
because it is related to parameter-free sub-gradient descent
algorithms~\citep{orabona13dimfree};  a similar potential function was
used for a parameter-free algorithm by~\citep{ChaudhuriFH09}.
The lower bound for this game was proven in \citet{streeter12unconstrained} for 1-dimensional spaces, and \citet{orabona13dimfree} extended it to Hilbert spaces and improved the leading constant. We report it here for completeness.
\begin{theorem}
\label{theo:lower}
Fix a non-trivial Hilbert space $\Hilbert$ and a specific online learning algorithm.
If the algorithm guarantees a zero regret against the competitor with zero norm, then there exists a sequence of $T$ cost vectors in $\Hilbert$, such that the regret against any other competitor is $\Omega(T)$.
On the other hand, if the algorithm guarantees a regret at most of $\epsilon>0$ against the competitor with zero norm, then, for any $0<\eta<1$, there exists a $T_0$ and a sequence of $T\geq T_0$ unitary norm vectors $g_t \in \Hilbert$, and a vector $u \in \Hilbert$ such that
\[
\Regret(u)  \geq (1-\eta) \|u\| \sqrt{\frac{1}{\log 2}}\sqrt{T \log \frac{\eta \|u\| \sqrt{T}}{3 \epsilon} } -2 ~.
\]
\end{theorem}

\subsection{Deriving a known-$T$ algorithm with minimax rates via the Normal approximation}
Consider the game with fixed known $T$, an adversary that plays from
$\G = \{g \in \Hilbert \mid \norm{g} \le G\}$, and
\begin{equation*}
  B(\ng) = \epsilon \bef{\norm{\ng}^2}{2 a T},
\end{equation*}
for constants $a > 1$ and $\epsilon>0$.  We will show that we are
in the case of the parallel adversary, Section~\ref{sec:parallel_adv}.
Both computing the $f_t$ based on Rademacher expectations and
evaluating the sufficient condition for those $f_t$ appear quite
difficult, so we turn to the Normal approximation. We then have
\[
\hf_t(x)
    = \E_{\phi_\tau}\left[\epsilon \bef{(x + \phi_\tau G)^2}{2 a t}\right]
    = \epsilon \left(1 - \frac{\pi G^2 (T - t)}{2a T}\right)^{-\frac{1}{2}}
             \bef{x^2}{2 a T - \pi G^2 (T- t)}~,
\]
where we have computed the expectation in a closed form for the second equality.
One can quickly verify that it satisfies the hypothesis of
Theorem~\ref{thm:normapprox} for $a > G^2 \pi/2$, hence $q_t(\ng) = \hf_t(\norm{\ng})$ will
be an admissible relaxation. Thus, by Corollary~\ref{cor:admbench}, we
immediately have
\[
  \Regret(u) \le B^*(\ng_T) + \epsilon \left(1 - \frac{\pi G^2}{2 a}\right)^{-\frac{1}{2}},
\]
and so by Lemma~\ref{lemma:bound_fenchel} in the Appendix, we can state the following Theorem, that matches the lower bound up to a constant multiplicative factor.
\begin{theorem}\label{thm:expgknownT}
Let $a > G^2 \pi/2$, and $\G=\{g: \|g\| \leq G\}$. Denote by $\hat{\ng}=\frac{\ng}{\norm{\ng}}$ if $\norm{\ng} \neq 0$, and $\vec{0}$ otherwise.
Fix the number of rounds $T$ of the game, and consider the strategy
\[
w_{t+1} = \epsilon \hat{\ng_t}
\frac{\bef{(\norm{\ng_t}+G)^2}{2 a T - \pi G^2 (T- t-1)}-\bef{(\norm{\ng_t}-G)^2}{2 a T - \pi G^2 (T- t-1)}}{2 G \sqrt{1 - \frac{\pi G^2 (T - t-1)}{2a T}}}~.
\]
Then, for any sequence of linear costs $\{g_t\}_{t=1}^T$, and any $u \in \Hilbert$, we have
\[
\Regret(u) \leq \norm{u} \sqrt{ 2 a T \log \left(\frac{\sqrt{a T}\norm{u}}{\epsilon}+1\right) } + \epsilon\left(\left(1 - \frac{\pi G^2}{2 a}\right)^{-\frac{1}{2}}-1\right)~.
\]
\end{theorem}

\subsection{AdaptiveNormal: an adaptive algorithm for unknown $T$}
 \label{sec:inc}
 Our techniques suggest the following recipe for developing adaptive
 algorithms: analyze the known $T$ case, define a potential $q_t(\ng)
 \approx V_T(\ng)$, and then analyze the incrementally-optimal
 algorithm for this potential \eqref{eq:mmq} via Theorem~\ref{thm:rrdual}.  We
 follow this recipe in the current section.
Again consider the game where an adversary that plays from
$\G = \{g \in \Hilbert \mid \norm{g} \le G\}$.
Define the function $f_t$ as
\[
f_t(x) = \beta_t \bef{x}{2 a t},
\]
where $a>\frac{3 \pi G^2}{4}$, and the $\beta_t$ is a decreasing sequence that will be specified in the following.  From this, we define the potential
$
  \q_t(\ng) = f_t(\norm{\ng}^2).
$
Suppose we play the incrementally-optimal algorithm of \eqref{eq:mmq}. Using Lemma~\ref{lem:mm:four} we can write the minimax value for the one-round game,
\begin{align*}
 \opteps_t(\ng_t)
    &= \E_{r \sim \{-1, 1\}} [f_{t+1}((\norm{\ng_t}+rG)^2)] - \q_t(\ng_t)\\
    &\leq \E_{\phi \sim N(0, \sigma^2)} [f_{t+1}((\norm{\ng_t} + \phi G)^2)] - \q_t(\ng_t)~.
      && \text{Lemma~\ref{lemma:normap}.}
\end{align*}
Using Lemma~\ref{lemma:argmax_diff} in the Appendix and our hypothesis on $a$, we have that the RHS of this inequality is maximized for $\|\ng_t\|=0$.
Hence, using the inequality $\sqrt{a+b}\leq \sqrt{a}+\frac{b}{2\sqrt{a}}, \forall a,b>0$, we get
\[
\opteps_t(\theta_t)
\le \beta_{t+1} \sqrt{1+\frac{\pi G^2}{2 a\,(t+1)- \pi G^2}} - \beta_t
\le \frac{\beta_t}{2} \frac{\pi G^2}{2 a\,(t+1)-\pi G^2} \leq \frac{\pi G^2 \beta_t}{4 a\,t}~.
\]
Thus, choosing $\beta_t = \epsilon/\log^2(t+1)$, for example, is sufficient to prove that $\opteps_{1:T}$ is bounded by $\epsilon \frac{\pi G^2}{a}$~\citep{baxley1992euler}.
Hence, again using Corollary~\ref{cor:admbench} and Lemma~\ref{lemma:bound_fenchel} in the Appendix, we can state the following Theorem.
\begin{theorem}\label{thm:expgadapt}
Let $a > 3 G^2 \pi/4$, and $\G=\{g: \|g\| \leq G\}$. Denote by $\hat{\ng}=\frac{\ng}{\norm{\ng}}$ if $\norm{\ng} \neq 0$, and $\vec{0}$ otherwise. Consider the strategy
\[
w_{t+1} = \epsilon \hat{\ng_t}
\left(\bef{(\norm{\ng_t}+G)^2}{2 a (t+1)}-\bef{(\norm{\ng_t}-G)^2}{2 a (t+1)}\right)\left(2 G \log^2(t+2)\right)^{-1}~.
\]
Then, for any sequence of linear costs $\{g_t\}_{t=1}^T$, and any $u \in \Hilbert$, we have
\[
\Regret(u) \leq \norm{u} \sqrt{ 2 a T \log \left(\frac{\sqrt{a T }\norm{u} \log^2(T+1)}{\epsilon}+1\right) } + \epsilon \left(\frac{\pi G^2}{a} -1\right)~.
\]
\end{theorem}

\coltonly{%
\acks{We thank Jacob Abernethy for many useful conversations about this work.}
}
\begin{small}
\setlength{\bibsep}{5pt}
\bibliography{main}
\end{small}
\clearpage
\appendix

\section{Proofs}
\subsection{Proof of Theorem~\ref{thm:rrdual}}
\begin{proof}
  Suppose the algorithm provides the reward guarantee \eqref{eq:rewb}.  First, note
  that for any comparator $u$, by definition we have
  \begin{equation}\label{eq:rewregdef}
    \Regret(u) = -\Reward - \langle g_{1:T}, u \rangle ~.
  \end{equation}
  Then, applying the definitions of Reward, Regret, and the Fenchel
  conjugate, we have
  \begin{align*}
    \Regret(u)
    &= \ng_T \cdot u - \Reward  && \text{By \eqref{eq:rewregdef}} \\
    &\leq \ng_T \cdot u - \q_T(\ng_T) + \approxeps_{1:T} && \text{By assumption \eqref{eq:rewb}}\\
    &\leq \max_\ng \big(\ng \cdot u - \q_T(\ng) + \approxeps_{1:T} \big)\\
    &= \q_T^*(u) + \approxeps_{1:T}~.
  \end{align*}
  For the other direction, assuming \eqref{eq:regb}, we have for any
  comparator $u$,
  \begin{align*}
    \Reward
    &= \theta_T \cdot u - \Regret(u) && \text{By \eqref{eq:rewregdef}}\\
    &= \max_{v} \big(\ng \cdot v - \Regret(v) \big) \\
    &\geq \max_{v} \big(\ng \cdot v - \q_T^*(v) - \approxeps_{1:T}\big)
      && \text{By assumption \eqr{regb}} \\
    &= \q_T(\ng) - \approxeps_{1:T}~.
  \end{align*}
  Alternatively, one can prove this from the Fenchel-Young
  inequality.
\end{proof}

\subsection{Proof of Lemma~\ref{lemma:fenchel_norm}}
\begin{proof}
  We have $B^*(u) = \sup_\ng \ang{u, \ng} - f(\norm{\ng})$.
  If $\norm{u}=0$, the stated equality is correct, in fact
  \[
    B^*(u) = \sup_\ng - f(\norm{\ng})=\sup_{\alpha \geq 0} - f(\alpha)=\sup_{\alpha \in \R} - f(\alpha)=f^*(0)~.
  \]
  Hence we can assume $\norm{u}\neq 0$, and by
  inspection we can take $\ng = \alpha u / \norm{u}$, with $\alpha\geq0$, and so
  \[
  B^*(u) = \sup_{\alpha \geq 0}  \alpha \norm{u} - f(\alpha) = \sup_{\alpha \in \R}  \alpha \norm{u} - f(\alpha)= f^*(\norm{u})~.
  \]
\end{proof}

\subsection{Proof of Theorem~\ref{thm:exact_minmax}}
\begin{proof}
  First we show that if $f$ satisfies the condition on the derivatives, the same conditions is satisfied by $f_t$, for all t.
  We have that all the $f_t$ have the form $h(x)=f(\sqrt{x^2+a})$, where $a\geq0$.
  Hence we have to prove that $\frac{x h''(x)}{h'(x)}\leq 1$.
  We have that $h'(x)=\frac{x f'(\sqrt{x^2+a})}{\sqrt{x^2+a}}$, and $h''(x)=\frac{x^2 f''(\sqrt{x^2+a}) + \frac{a}{\sqrt{x^2+a}} f'(\sqrt{x^2+a})}{x^2+a}$, so
  \[
    \frac{x h''(x)}{h'(x)} = \frac{x^2 f''(\sqrt{x^2+a})\sqrt{x^2+a} }{f'(\sqrt{x^2+a}) (x^2+a)} + \frac{a}{x^2+a} \leq \frac{x^2}{x^2+a} + \frac{a}{x^2+a}=1,
  \]
  where in the inequality we used the hypothesis on the derivatives of $f$.

  We show $V_t$ has the stated form by induction from $T$ down to 0.  The base case
  for $t=T$ is immediate.  For the induction step, we have
  \begin{align*}
    V_t(\ng) 
      &= \min_w \max_g \ang{w, g} + V_{t+1}(\ng -g) && \text{Defn.}\\
      &= \min_w \max_g \ang{w, g} + f_{t+1}\left(\norm{\ng -g}\right)
         && \text{(IH)}\\
      &= f_{t+1}\left(\sqrt{\norm{\ng}^2 + G^2}\right) 
         && \text{Assumption \eqref{eq:strong}}\\
      &= f\left(\sqrt{\norm{\ng}^2 + G^2 (T -t)}\right)~.
  \end{align*}
  The sufficient condition for \eqref{eq:strong} follow immediately from
  Lemma~\ref{lem:mm:three}.
\end{proof}

\subsection{Proof of Theorem~\ref{thm:normapprox}}
\begin{proof}
  First, we need to show the functions $f_t$ and $\hf_t$ of
  \eqref{eq:def_ht} are even.  Let $r$ be a random variable draw from
  any symmetric distribution.  Then, we have
  \[f_t(x) 
     = \E[f(\abs{x+r})] 
     = \E[f(\abs{-x-r})] 
     = \E[f(\abs{-x+r})] = f_t(-x),
  \]
  where we have used the fact that $\abs{\cdot}$ is even and the symmetry of
  $r$.

  We show $f_t(\norm{\ng}) = V_t(\ng)$ inductively from $t=T$ down to
  $t=0$.  The base case $T=t$ follows from the definition of $f_t$.
  Then, suppose the result holds for $t+1$.  We have
  \begin{align*}
    V_t(\ng) 
    &= \min_w \max_{g\in \G} \ang{w, g} + V_{t+1}(\ng - g) 
      && \text{Defn.} \\
    &= \min_w \max_{g \in \G} \ang{w, g} + f_{t+1}(\norm{\ng - g}) 
      && \text{IH} \\
    &= \E_{r \sim \Rad} \big[f_{t+1}(\norm{\ng} + rG)\big]
      && \text{Lemma \ref{lem:mm:four}} \\
    &= \E_{r \sim \Rad} \left[
       \E_{r_{\tau-1} \sim \Rad^{\tau-1}} \big[
         f(\norm{\ng} + rG + r_{\tau-1} G)\big]\right]\\
    &= f_t(\norm{\ng}),
  \end{align*}
  where the last two lines follow from the definition of $f_t$ and
  $f_{t+1}$.
 The case for $\hf_t$ is similar, using the hypothesis of the Theorem we have
\begin{align*}
\min_w \max_{g \in \G} \ \langle g, w \rangle + \hf_{t+1}(\norm{\ng - g}) 
  &= E_{r \sim \{-1,1\}} [\hf_{t+1}(\|\ng\| + r G)] \leq E_{r \sim N(0,\sigma^2)} [\hf_{t+1}(\|\ng\| + \phi G)] \\
&= \hf_{t}(\|\ng\|),
\end{align*}
where in the inequality we used Lemma~\ref{lemma:normap}, and in the second equality the definition of $\hf_{t}$. Hence, $\q_t(\ng) = \hf_t(\norm{\ng})$ satisfy \eqref{eq:repst} with $\approxeps_{t} = 0$.
Finally, the sufficient conditions come immediately from
Lemma~\ref{lem:mm:four}.
\end{proof}

\subsection{Analysis of the one-round game: Proofs of Lemmas~\ref{lem:mm:three} and \ref{lem:mm:four}}\label{sec:lemproof}
\newcommand{\wperp}{\hat{w}}

In the process of proving these lemmas, we also show the following
general lower bound:
\begin{lemma}\label{lem:mm:two}
Under the same definitions as in Lemma~\ref{lem:mm:three}, if $d > 1$, we have
\[
H \geq  h\left(\sqrt{\|\theta\|^2 + G^2}\right)~.
\]
\end{lemma}

We now proceed with the proofs.  The $d=1$ case for Lemma~\ref{lem:mm:four}
was was proved in~\citet{mcmahan13minimax}.

Before proving the other results, we simplify a bit the formulation of the minimax problem.
For the other results, the maximization wrt $g$ of a convex function is always attained when $\|g\|=G$.
Moreover, in the case of $\norm{\ng}=0$ the other results are true, in fact
\[
\min_{w} \max_{g \in \G} \ \langle w, g \rangle + h(\|\theta - g\|) =  \min_{w} \max_{\norm{g}=G} \ \langle w, g \rangle + h(\|g\|) = \min_{w} G\norm{w} + h(G) = h(G)~.
\]
Hence, without loss of generality, in the following we can write $w=\alpha \frac{\theta}{\|\theta\|}+\wperp$, where $\ang{\wperp,  \theta}=0$. It is easy to see that in all the cases the optimal choice of $g$ turns out to be $g=\beta \frac{\theta}{\|\theta\|}+\gamma \wperp$, where $\gamma\geq0$. With these settings, the minimax problem is equivalent to
\[
\min_{w} \max_{g \in \G} \ \langle w, g \rangle + h(\|\theta - g\|) =  \min_{\alpha, \wperp} \max_{\beta^2 + \norm{\wperp}^2 \gamma^2=G^2} \  \alpha \beta+ \gamma \|\wperp\|^2 + h(\sqrt{\|\theta\|^2 -2 \beta \|\theta\| + G^2})~.
\]
By inspection, the player can always choose $\wperp = \mathbf{0}$ so $\gamma \|\wperp\|^2=0$. Hence we have a simplified and equivalent form of our optimization problem
\begin{equation}
\label{eq:simple_minmax}
\min_{w} \max_{g \in \G} \ \langle w, g \rangle + h(\|\theta - g\|) =  \min_{\alpha} \max_{\beta^2 \leq G^2} \  \alpha \beta + h(\sqrt{\|\theta\|^2 -2 \beta \|\theta\| + G^2})~.
\end{equation}

For Lemma~\ref{lem:mm:two},  it is enough to set $\beta=0$ in \eqref{eq:simple_minmax}.

For Lemma~\ref{lem:mm:three},  we upper bound the minimum wrt to $\alpha$ with the specific choice of $\alpha$.
In particular, we set $\alpha=\frac{\norm{\ng}}{\sqrt{\|\theta\|^2 + G^2}} h'\left(\sqrt{\|\theta\|^2 + G^2}\right)$ in \eqref{eq:simple_minmax}, and get
\begin{align*}
\min_{w} \max_{g \in \G} \ \langle w, g \rangle + h(\|\theta - g\|) 
\leq \max_{\beta^2 \leq G^2} \ \frac{\beta \|\theta\| h'(\sqrt{\|\theta\|^2 + G^2})}{\sqrt{\|\theta\|^2 + G^2}} + h(\sqrt{\|\theta\|^2 -2 \beta \|\theta\| + G^2})~.
\end{align*}
The derivative of argument of the max wrt $\beta$ is
\begin{align}
\frac{\|\ng\| h'\left(\sqrt{\|\ng\|^2 + G^2}\right)}{\sqrt{\|\ng\|^2 + G^2}} - \frac{\|\ng\| h'\left(\sqrt{\|\ng\|^2 -2 \beta \|\ng\| + G^2}\right)}{\sqrt{\|\ng\|^2 -2 \beta \|\ng\| + G^2}}~. \label{eq:deriv_f}
\end{align}
We have that if $\beta=0$ the first derivative is 0.
Using the hypothesis on the first and second derivative of $h$, we have that the second term in \eqref{eq:deriv_f} increases in $\beta$. Hence $\beta=0$ is the maximum. Comparing the obtained upper bound with the lower bound in Lemma~\ref{lem:mm:two}, we get the stated equality.

For Lemma~\ref{lem:mm:four}, the second derivative wrt $\beta$ of the argument of the minimax problem in \eqref{eq:simple_minmax} is
\[
-\|\ng\| \frac{-\|\ng\|h''( \sqrt{\|\ng\| + G^2 - 2 \beta \|\ng\|}) + h'( \sqrt{\|\ng\| + G^2 - 2 \beta \|\ng\|}) \frac{\|\ng\|}{\sqrt{\|\ng\| + G^2 - 2 \beta \|\ng\|}}}{\|\ng\| + G^2 - 2 \beta \|\ng\|}
\]
that is non negative, for our hypothesis on the derivatives of $h$. Hence, the argument of the minimax problem is convex wrt $\beta$, hence the maximum is achieved at the boundary of the domains, that is $\beta^2=G^2$. So, we have
\[
\min_{w} \max_{g \in \G} \ \langle w, g \rangle + h(\|\ng - g\|) = \max \left(-G \alpha + h(\|\ng\| + G), G \alpha + h(|\|\ng\| - G|) \right)~.
\]
The argmin of this quantity wrt to $\alpha$ is obtained when the the
two terms in the max are equal, so we obtained the stated equality.

\subsection{Lemma~\ref{lemma:bound_fenchel}}
\begin{lemma}
\label{lemma:bound_fenchel}
Define $f(\ng) = \beta \exp{ \frac{\|\ng\|^2}{2 \alpha}}$, for $\alpha, \beta>0$. Then
\[
f^*(\bw) \leq \|w\|\sqrt{ 2 \alpha \log \left(\frac{\sqrt{\alpha} \|w\|}{\beta}+1\right) } -\beta~.
\]
\end{lemma}
\begin{proof}
From the definition of Fenchel dual, we have
\begin{align*}
f^*(w)= \max_{\ng} \langle \ng, w \rangle - f(\ng) \leq \langle \ng^*, w \rangle -\beta~.
\end{align*}
where $\ng^*= \argmax_{\ng} \langle \ng, w \rangle - f(\ng)$. We now use the fact that $\ng^*$ satisfies $w = \nabla f(\ng^*)$, that is
\begin{align*}
w = \ng^*  \frac{\beta}{\alpha} \exp\left(\frac{\|\ng^*\|^2}{2 \alpha}\right),
\end{align*}
in other words we have that $\ng^*$ and $w$ are in the same direction. Hence we can set $\ng^*= q w$, so that $f^*(w) \leq q \norm{w}^2 - \beta$. We now need to look for $q>0$, solving
\begin{align*}
\frac{q \beta}{\alpha}\exp\left(\frac{\norm{w}^2 q^2}{2 \alpha}\right)=1
&\Leftrightarrow \frac{\norm{w}^2 q^2}{2 \alpha} + \log \frac{\beta q}{\alpha} =0 \\
&\Leftrightarrow q = \sqrt{\frac{2 \alpha}{\norm{w}^2} \log \frac{\alpha}{q \beta}} = \sqrt{\frac{2 \alpha}{\norm{w}^2} \log \left( \frac{\sqrt{\alpha} \norm{w}}{\beta \sqrt{2 \log \frac{\alpha}{q \beta} }}\right) }~.
\end{align*}
Using the elementary inequality $\log x \leq \frac{m}{e} x^\frac{1}{m}, \forall m>0$, we have
\begin{align*}
&q^2 = \frac{2 \alpha}{\norm{w}^2} \log \frac{\alpha}{q \beta} \leq \frac{2 m \alpha}{e \norm{w}^2} \left(\frac{\alpha}{q \beta}\right)^\frac{1}{m} \\
&\Rightarrow q^{\frac{2m+1}{m}} \leq \frac{2 m \alpha^{\frac{m+1}{m}}}{e \norm{w}^2 \beta^\frac{1}{m}} \Rightarrow q \leq \left(\frac{2 m}{e \norm{w}^2 }\right)^\frac{m}{2m+1} \alpha^\frac{m+1}{2m+1} \beta^{-\frac{1}{2m+1}}\\
&\Rightarrow \frac{\alpha}{\beta q} \geq \left(\frac{e \norm{w}^2 }{2 m}\right)^\frac{m}{2m+1} \alpha^{1-\frac{m+1}{2m+1}} \beta^{\frac{1}{2m+1}-1}=\frac{\alpha}{\beta q} \geq \left( \sqrt{\frac{e}{2 m}} \frac{\norm{w} \sqrt{\alpha}}{\beta}\right)^{\frac{2m}{2m+1}}~.
\end{align*}
Hence we have
\begin{align*}
q \leq \sqrt{\frac{2 \alpha}{\norm{w}^2} \log \left( \frac{\sqrt{\alpha} \norm{w}}{\beta \sqrt{\frac{4m}{2m+1} \log \sqrt{\frac{e}{2 m}} \frac{ \norm{w} \sqrt{\alpha}}{\beta} }}\right) }~.
\end{align*}
We set $m$ such that $\sqrt{\frac{e}{2 m}} \frac{ \norm{w} \sqrt{\alpha}}{\beta} = \sqrt{e}$, that is $\frac{1}{2} \left(\frac{ \norm{w} \sqrt{\alpha}}{\beta}\right)^2 = m$. Hence we have $\log \sqrt{\frac{e}{2 m}} \frac{ \norm{w} \sqrt{\alpha}}{\beta} = \frac{1}{2}$ and $\frac{\frac{ \norm{w} \sqrt{\alpha}}{\beta}}{\sqrt{2m}}=1$, and obtain
\begin{align*}
f^*(w) \leq \norm{w} \sqrt{ 2 \alpha \log \left( \sqrt{\left(\frac{\sqrt{\alpha} \norm{w}}{\beta}\right)^2+1}\right) } -\beta
\leq \norm{w} \sqrt{ 2 \alpha \log \left(\frac{\sqrt{\alpha} \norm{w}}{\beta}+1\right) }-\beta~.
\end{align*}
\end{proof}

\subsection{Lemma~\ref{lemma:argmax_diff}}

\begin{lemma}
\label{lemma:diff_exp_decreasing}
Let $f(x)=b \exp\left(\frac{x^2}{a}\right)-\exp\left(\frac{x^2}{c}\right)$. If $a\geq c > 0$, $b\geq0$, and $b\, c \leq a$, then the function $f(x)$ is decreasing for $x\geq0$.
\end{lemma}
\begin{proof}
The proof is immediate from the study of the first derivative.
\end{proof}

\begin{lemma}
\label{lemma:fract_decreasing}
Let $f(t)=\frac{a^\frac{3}{2}\,t\,\sqrt{t+1}}{(a\,(t+1)-b)^\frac{3}{2}}$, with $a\geq 3/2 b>0$. Then $f(t)\leq 1$ for any $t\geq0$.
\end{lemma}
\begin{proof}
The sign of the first derivative of the function has the same sign of
\[
(2\,a-3b)(t+1)+b,
\]
hence from the hypothesis on $a$ and $b$ the function is strictly increasing. Moreover the asymptote for $t\rightarrow \infty$ is 1, hence we have the stated upper bound.
\end{proof}

\begin{lemma}
\label{lemma:argmax_diff}
Let $f_t(x)=\beta_{t} \bef{x}{2a t}$, $\beta_{t+1}\leq \beta_t,\  \forall t$. If $a\geq \frac{3 \pi G^2}{4}$, then
\[
\argmax_{x} \E_{\phi \sim N(0, \sigma^2)} [f_{t+1}(x + \phi G)] - f_t(x) = 0,
\]
where $\sigma^2=\frac{\pi}{2}$.
\end{lemma}
\begin{proof}
We have 
\[
\E_{\phi \sim N(0, \sigma^2)} [f_{t+1}(x + \phi G)] = \beta_{t+1} \sqrt{\frac{a\,(t+1)}{a\,(t+1)-\sigma^2 G^2}} \exp\left(\frac{x^2}{2\left[a\,(t+1)-\sigma^2 G^2\right]}\right),
\]
so we have to study the max of 
\[
\beta_{t+1} \sqrt{\frac{a\,(t+1)}{a\,(t+1)-\sigma^2 G^2}} \exp\left(\frac{x^2}{2\left[a\,(t+1)-\sigma^2 G^2\right]}\right) - \beta_{t} \exp\left(\frac{x^2}{2\,a\,t}\right)~.
\]
The function is even, so we have a maximum in zero iff the function is decreasing for $x>0$.
Observe that, from Lemma~\ref{lemma:fract_decreasing}, for any $t\geq0$
\[
\sqrt{\frac{a\,(t+1)}{a\,(t+1)-\sigma^2 G^2}} \frac{a\,t}{a\,(t+1)-\sigma^2 G^2} \leq 1~.
\]
Hence, using Lemma~\ref{lemma:diff_exp_decreasing}, we obtain that the stated result.
\end{proof}

\end{document}